\documentclass[arxiv,12pt]{colt2026}

\usepackage{times}
\usepackage{physics}

\title[Minimax Rates for Hyperbolic Hierarchical Learning]{Minimax Rates for Hyperbolic Hierarchical Learning}

\newcommand{\dT}{{d_T}} 
\newcommand{\rge}{\mathcal{E}_{R,\eta}^{\mathrm{reg}}} 
\newcommand{\FL}{\mathcal{F}_L} 
\newcommand{\acc}{\varepsilon} 
\newcommand{\Rk}{\mathbb{R}^k} 
\newcommand{\B}{\mathcal{B}} 
\newcommand{\Lip}{{\mathrm{Lip}}} 
\newcommand{\poly}{{\mathrm{poly}}} 
\newcommand{\fat}{{\mathrm{fat}}} 
\newcommand{\R}{\mathbb{R}} 
\newcommand{\err}{\epsilon} 
\newcommand{\Hkk}{\mathbb{H}_\kappa^k} 
\newcommand{\dH}{d_\mathbb{H}} 
\newcommand{\dHam}{d_{\mathrm{Ham}}} 
\newcommand{\VolH}{\operatorname{Vol}_{\mathbb H}} 

\coltauthor{%
 \Name{Divit Rawal} \Email{divit.rawal@berkeley.edu}\\
 \addr University of California, Berkeley
 \AND
 \Name{Sriram Vishwanath} \Email{sriram@ece.gatech.edu}\\
 \addr Georgia Institute of Technology%
}

\begin{document}

\maketitle

\begin{abstract}
We prove an exponential separation in sample complexity between Euclidean and hyperbolic representations for learning on hierarchical data under standard Lipschitz regularization. For depth-$R$ hierarchies with branching factor $m$, we first establish a geometric obstruction for Euclidean space: any bounded-radius embedding forces volumetric collapse, mapping exponentially many tree-distant points to nearby locations. This necessitates Lipschitz constants scaling as $\exp(\Omega(R))$ to realize even simple hierarchical targets, yielding exponential sample complexity under capacity control. We then show this obstruction vanishes in hyperbolic space: constant-distortion hyperbolic embeddings admit $O(1)$-Lipschitz realizability, enabling learning with $n = O(mR \log m)$ samples. A matching $\Omega(mR \log m)$ lower bound via Fano's inequality establishes that hyperbolic representations achieve the information-theoretic optimum. We also show a geometry-independent bottleneck: any rank-$k$ prediction space captures only $O(k)$ canonical hierarchical contrasts.
\end{abstract}


\section{Introduction}

Hierarchical structure appears throughout machine learning: taxonomies and ontologies, coarse-to-fine classification, latent tree models, and multiscale decision systems. A common empirical approach is to learn a representation and then fit a relatively simple readout, typically under some norm or Lipschitz control. In parallel, hyperbolic embeddings have become a standard geometric primitive for hierarchical data, motivated by the exponential volume growth of hyperbolic space and supported by strong empirical performance in representation learning and graph neural networks \citep{nickel_poincare_2017, ganea_hyperbolic_2018, chami_hyperbolic_2019}. Despite this, existing theoretical accounts are primarily geometric, e.g., distortion guarantees for embedding trees into $\mathbb H^k$ \citep{sarkar_low_2011, de_sa_representation_2018} and do not establish when and why representation geometry changes sample complexity under standard statistical regularization.

In this work, we show that Euclidean representations with bounded radii do not preserve intrinsically smooth hierarchical targets. We formalize intrinsic smoothness using a correlation metric on the tree (weighted path length), which arises naturally in latent tree models via correlation decay. A ``simple" target is defined to be Lipschitz in the correlation metric and a ``good" representation preserves intrinsic simplicity after embedding, meaning that a readout with a polynomially-bounded norm/Lipschitz constant can realize the target in the embedded space.

We study a top-down local refinement primitive: conditioned on a parent category, we predict the correct child among its siblings. This is precisely the decomposition used by hierarchical classification pipelines (e.g., local classifiers per parent and top-down inference), and matches how hierarchical performance is often reported (conditional on a superclass rather than globally). This refinement task is standard multiclass classification with classification noise: an example $(i,c)$ consists of a parent index and a child label, with Bayes label $\mathbb{I}\{c=\theta_i\}$ observed through a $\mathrm{BSC}(\rho)$. Hence the protocol induces PAC learning over the finite concept class $\mathcal{H}_{\mathrm{path}}$, and lower bounds proved under
a strictly stronger observation model apply a fortiori to any learner.

\paragraph{Contributions.}
We show that geometry induces exponential separations in hierarchical learning:

\begin{itemize}
\item \textbf{Euclidean lower bound:} bounded-radius Euclidean embeddings force $\exp(\Omega(R))$ Lipschitz constants and exponential
sample complexity.

\item \textbf{Low-rank expressivity:} rank-$k$ predictors capture only $O(k)$ out of $m^R$ hierarchical contrasts, ruling out low-rank
compression.

\item \textbf{Hyperbolic optimality:} hyperbolic embeddings enable realizability and achieve the minimax rate $\Theta(mR\log m)$, matching an information-theoretic lower bound.
\end{itemize}

\paragraph{Organization.}
Section~\ref{sec:model-and-setup} introduces the model and the canonical refinement task.
Section~\ref{sec:euclidean-barrier} establishes the Euclidean sample complexity lower bound.
Section~\ref{sec:low-rank-representation-barrier} proves the low-rank expressivity bound.
Section~\ref{sec:hyperbolic-achievability-and-optimality} shows hyperbolic minimax optimality.
Section~\ref{sec:conc} concludes.

\section{Model and Setup}\label{sec:model-and-setup}

This section introduces the mathematical framework for our analysis. We define the tree structure and intrinsic correlation metric (Definition \ref{def:rooted-tree-with-weighted-metric}), formalize the hypotheses classes under Lipschitz regularization (Definition \ref{def:lipschitz-hypothesis-class}), specify what constitutes a good representation (Definition \ref{def:representation-and-bi-lipschitz-embedding}), and present the canonical learning protocol (Section~\ref{subsec:learning-protocol}).

Although all results in the main text are derived for trees, they extend to hierarchies with tree-like structures. This is established via a quasi-isometry between trees and trees with sparse lateral edges; details in Appendix~\ref{app:beyond-trees-high-temperature-extensions}.

\begin{definition}[Rooted tree with weighted metric]\label{def:rooted-tree-with-weighted-metric}
	Let $T = (V,E)$ be a rooted tree of depth $R$ with root $r$. Define:
\begin{itemize}
\item $V_{\leq R}$: the set of all nodes at depth at most $R$;
\item $L_R \subseteq V_{\leq R}$: the set of leaves (nodes at depth exactly $R$);
\item $V_{\leq R}\setminus L_R$: the set of internal nodes (non-leaves).
\end{itemize}
For nodes $u,v\in V_{\leq R}$, let $\mathrm{path}(u,v)$ denote the unique simple path between them, and let $\abs{\mathrm{path}(u,v)}$ denote the number of edges on this path.

Assign each edge $e\in E$ a positive weight $w_e>0$. The \emph{correlation metric} is the weighted tree distance:
	\begin{equation*}
		d_\mathrm{corr}(u,v) = \sum_{e \in \mathrm{path}(u,v)} w_e.
	\end{equation*}
	In the homogeneous case, we set $w_e \equiv \lambda$ for all edges, which yields $d_\mathrm{corr}(u,v) = \lambda \abs{\mathrm{path}(u,v)}$. It is straightforward to verify that $d_\mathrm{corr}$ is a metric on $V_{\leq R}$.
\end{definition}

The metric $d_\mathrm{corr}$ formalizes an intrinsic notion of similarity for the hierarchy. We show that targets are ``simple" or ``smooth" if they are Lipschitz in the intrinsic metric, and that representations are ``good" if that simplicity remains after embedding.

\textbf{Regular growth event}. Our results hold on the regular growth event $\rge$, under which the tree exhibits typical Galton-Watson growth: $\abs{L_R} \asymp m^R$ and no single mid-depth subtree dominates. This event has a probability $1 - o(1)$ under standard moment conditions; see Appendix~\ref{app:probability-of-the-regular-growth-event} for a formal definition and concentration bounds.

\begin{definition}[Lipschitz hypothesis class]\label{def:lipschitz-hypothesis-class}
	For $L > 0$, define the Lipschitz hypothesis class:
	\begin{equation*}
		\FL \doteq \big\{f : V \to [-1,+1] \big\vert \abs{f(u) - f(v)} \leq L d_\mathrm{corr}(u,v) \big\}.
	\end{equation*}
\end{definition}

The role of the Lipschitz class $\FL$ is twofold. First, it formalizes the notion that ``hierarchical targets are smooth in the intrinsic (tree/correlation) metric,'' which is the baseline against which our approximation bounds are stated. Second, the concrete targets used throughout the paper are instances of such intrinsic-smooth functions. While our hyperbolic achievability results (Section~\ref{sec:hyperbolic-achievability-and-optimality}) are presented for a specific finite hypothesis class, the geometric obstructions we prove extend to representing arbitrary intrinsically Lipschitz hierarchical structures.

\paragraph{Remark on Lipschitz control.} We focus on Lipschitz (equivalently, norm-bounded) hypothesis classes because they capture the implicit bias of modern high-dimensional learning. Standard regularization techniques such as weight decay in neural networks \citep{DBLP:journals/corr/BartlettFT17}, margin maximization in SVMs \citep{Cortes1995}, and RKHS norm penalties \citep{10.7551/mitpress/4175.001.0001}, effectively constrain the Lipschitz constant of the learned predictor with respect to the representation. Additionally, standard generalization bounds (e.g. Rademacher complexity) scale with the Lipschitz constant \citep{Bartlett_2005, von_luxburg_classification_2004}. Thus, capacity control via Lipschitz or norm constraints is both theoretically well motivated and empirically used.

\begin{definition}[Representation and bi-Lipschitz embedding]\label{def:representation-and-bi-lipschitz-embedding}
	A representation is a map $\phi : V \to (\mathcal X,d_{\mathcal X})$. We refer to $\phi$ as $(s,D)$-bi-Lipschitz if for all $u,v \in V$:
	\begin{equation*}
		\frac{s}{D} d_\mathrm{corr}(u,v) \leq d_{\mathcal X}(\phi(u), \phi(v)) \leq sD d_\mathrm{corr}(u,v).
	\end{equation*}
	If $s=1$, then we say $\phi$ is $D$-bi-Lipschitz.
\end{definition}

\paragraph{Preservation of Lipschitz structure.}
If a map $\phi$ is $(s,D)$-bi-Lipschitz from $(V_{\leq R}, d_{\mathrm{corr}})$ into $(\mathcal{X}, d_{\mathcal{X}})$, then Lipschitz functions on the tree remain Lipschitz after embedding, with at most a factor-$D$ increase in the Lipschitz constant (see Lemma~\ref{lem:lipschitz-preservation-appendix} in Appendix~\ref{app:lipschitz-preservation} for the formal statement). This is why bi-Lipschitz embeddings with bounded distortion are ``good'' representations for hierarchical datasets.

We will largely focus on sample complexity in the standard PAC learning framework. Let $\mathcal{P}$ be a class of probability distributions over $V_{\leq R} \times \{0,1\}$ (or, more generally, over $\mathcal{X}\times\{0,1\}$ for a representation $\phi:V_{\leq R}\to\mathcal{X}$). For a hypothesis class $\mathcal{H}$ and target function $g$, define the risk of $h\in\mathcal{H}$ under distribution $P\in\mathcal{P}$ as:
\begin{equation*}
\mathrm{Risk}_P(h, g) \doteq \mathbb{E}_{(x,y)\sim P}\left[|h(x) - g(x)|\right].
\end{equation*}

\begin{definition}[Sample complexity]\label{def:sample-complexity}
For a hypothesis class $\mathcal{H}$, target class $\mathcal{G}$, distribution class $\mathcal{P}$, accuracy $\varepsilon>0$, and confidence $\delta\in(0,1)$, the \emph{sample complexity} $n(\mathcal{H}, \mathcal{G}, \mathcal{P}, \varepsilon, \delta)$ is the minimum $n$ such that there exists a learning algorithm $\mathcal{A}$ satisfying: for all $g\in\mathcal{G}$ and all $P\in\mathcal{P}$, given $n$ i.i.d.\ samples from $P$, the algorithm outputs $\hat{h}\in\mathcal{H}$ such that
\begin{equation*}
\Pr\left[\mathrm{Risk}_P(\hat{h}, g) \leq \varepsilon\right] \geq 1 - \delta.
\end{equation*}
When the context is clear, we abbreviate this as $n(\varepsilon, \delta)$.
\end{definition}

For our results, we focus on the worst-case sample complexity over the specified distribution and target classes, which lower-bounds the complexity of any learner.

\paragraph{Convention and remarks.}Throughout, we treat the branching factor ($m$ for deterministic $m$-ary trees, $\xi$ for Galton-Watson trees) as a fixed constant. All asymptotic statements are in the depth $R$. Weighted tree metrics also arise naturally from correlation decay in tree-structured graphical models; see Appendix~\ref{app:ising-model}. In hierarchical generative models, the information-theoretic distance between nodes (defined via correlation decay) is exactly proportional to $\dT$. Thus, learning functions that are Lipschitz with respect to $\dT$ corresponds to learning functions that vary smoothly with the generative statistics of the data.

\subsection{Canonical Learning Protocol}\label{subsec:learning-protocol}

Hierarchical prediction problems are typically \emph{coarse-to-fine}: to make a fine-grained decision, one conditions on having already identified the relevant coarse category. Concretely, most tasks on a rooted tree: leaf classification, taxonomy navigation, subtree retrieval, and multiscale regression~/~classification; can be decomposed into a sequence of local refinement decisions of the form \emph{``within the current prefix subtree, which child leads towards the correct refinement?"}. We formalize this primitive via the following learning protocol, which makes the sampling distribution and information constraints explicit.

\paragraph{Unknown parameter (latent root-to-leaf path).} 
Let $\theta=(\theta_1,\dots,\theta_R)\in[m]^R$ denote an unknown root-to-leaf path in the complete $m$-ary tree of depth $R$, where $\theta_i\in[m]$ specifies the child index at depth $i$. For $i\in\{0,1,\dots,R\}$, define the depth-$i$ prefix node
\begin{equation*}
	u_i(\theta) \doteq (\theta_1,\dots,\theta_i)\in [m]^i,
\end{equation*}
with the convention $u_0(\theta)=r$ denoting the root. For any leaf $v=(v_1,\ldots,v_R)\in L_R$, write $\mathrm{ch}_i(v)=v_i\in[m]$ for the child index at depth $i$ along the root-to-$v$ path. For any node $u\in V_{\leq R}$, denote by $L_R^{(u)}\subseteq L_R$ the set of leaves in the subtree rooted at $u$.

\paragraph{Data generation (local refinement query).}
Each training example is obtained by selecting a query depth uniformly, then sampling a leaf uniformly from the subtree consistent with the latent path up to the parent of that depth. Formally:
\begin{enumerate}
\item Draw a query depth $I \sim \mathrm{Unif}([R])$, where $[R]=\{1,\ldots,R\}$.
\item Conditional on $I=i$, draw a leaf uniformly from the subtree rooted at the on-path prefix $u_{i-1}(\theta)$:
\begin{equation*}
	V \mid (I=i) \;\sim\; \mathrm{Unif}\bigl(L_R^{(u_{i-1}(\theta))}\bigr).
\end{equation*}
Equivalently, $V$ is uniformly distributed over leaves $(v_1,\ldots,v_R)\in L_R$ satisfying $(v_1,\ldots,v_{i-1})=u_{i-1}(\theta)$.
\item The Bayes-optimal label indicates whether $V$ lies in that particular child-subtree at depth $i$:
\begin{equation*}
	Y^*(V,I) \doteq \mathbb{I}\{\mathrm{ch}_I(V)=\theta_I\}\in\{0,1\}.
\end{equation*}
\item The observed label $Y$ is corrupted by a binary symmetric channel (BSC) with crossover probability $\rho\in(0,1/2)$:
\begin{equation*}
	\Pr[Y=1\mid V=v,I=i,\theta] = \begin{cases}
	1-\rho & \text{if }\mathrm{ch}_i(v)=\theta_i,\\
	\rho & \text{if }\mathrm{ch}_i(v)\neq \theta_i.
	\end{cases}
\end{equation*}
\end{enumerate}
The learner observes $\bigl(\phi(V),I,Y\bigr)$, where $\phi:V_{\leq R}\to\mathcal{X}$ is the representation map (Definition~\ref{def:representation-and-bi-lipschitz-embedding}), and outputs a predictor $\hat{y}:\mathcal{X}\times[R]\to\{0,1\}$.

\paragraph{PAC via a stronger oracle model.}
The protocol induces a standard supervised learning problem with classification noise. In the more informative oracle model of Appendix~\ref{app:fano-details}, each example is an i.i.d. draw $(X,Y)$ where $X=(I,C)\in [R]\times[m]$ with $I\sim\mathrm{Unif}([R])$ and $C\sim\mathrm{Unif}([m])$, and the Bayes label is $Y^*(I,C)=\mathbb{I}\{C=\theta_I\}$ observed through a $\mathrm{BSC}(\rho)$. Equivalently, we are PAC-learning the finite concept class $\mathcal{H}_{\mathrm{path}}=\{h_\theta(i,c)=\mathbb{I}\{c=\theta_i\}\mid\theta\in[m]^R\}$ under classification noise \citep{Valiant1984ATO, Angluin1988LearningFN}. Since the oracle reveals $(I,C)$ whereas the original protocol reveals only $(I,\phi(V))$, the oracle model is strictly more informative; hence any information-theoretic lower bound proved in the oracle model applies a fortiori to any learner (proper or improper) in our original setting.

\paragraph{Information structure.}
Conditional on $I=i$, the sample is informative only about $\theta_i$, and only a $1/m$ fraction of leaves carry the distinguishing signal, observed through a $\mathrm{BSC}(\rho)$. This yields a minimax lower bound of $\Omega(mR\log m)$ via Fano's inequality (Appendix~\ref{app:fano-details}).

\section{Euclidean Sample Complexity Lower Bound}
\label{sec:euclidean-barrier}

We show that bounded-radius Euclidean embeddings cannot realize intrinsically smooth hierarchical targets (Lipschitz in $d_{\mathrm{corr}}$). The argument has two steps:
\begin{enumerate}
    \item \textbf{Collision.} Any embedding of $\abs{L_R}\asymp m^R$ leaves into a $k$-dimensional ball of radius $B$ maps some pair of leaves at tree distance $\Omega(R)$ to Euclidean distance $O\left(B\exp(-cR/k)\right)$ (Lemma~\ref{lem:deep-crushing}). This is a volumetric pigeonhole argument: Euclidean covering numbers grow polynomially while $\abs{L_R}$ grows exponentially.
    \item \textbf{Approximation lower bound.} The collided pair induces a subtree-indicator target that is $O(1/R)$-Lipschitz in $d_{\mathrm{corr}}$ but forces any Euclidean predictor with $\mathrm{Lip}_{\norm{\cdot}_2}(h)\leq \mathrm{poly}(R)$ to incur error, i.e. realizability requires $\mathrm{Lip}_{\norm{\cdot}_2}(h) = \Omega(\exp(R/k)/B)$ (Theorem~\ref{thm:lipschitz-capacity-barrier}).
\end{enumerate}

Throughout this section, assume that edge weights satisfy $w_e \geq \lambda$ for all $e\in E$.

We begin with the geometric collision statement. It isolates the only property of the tree we use: under $\rge$, the number of leaves grows as $\asymp m^R$, while the covering number of a Euclidean ball grows only polynomially with resolution.

\begin{lemma}[Volumetric collapse]\label{lem:deep-crushing}
	On $\rge$, let $\phi: L_R \to \Rk$ satisfy $\phi(L_R) \subseteq \B_2 (0,B)$. Then, there exist leaves $u^*,v^* \in L_R$ with:
	\begin{equation*}
		\norm{\phi(u^*) - \phi(v^*)}_2 \leq cB \exp\left(-\frac{\log m - 4\eta}{2k}R\right), \quad \text{while } d_\mathrm{corr}(u^*,v^*) \geq \lambda R
	\end{equation*}
	where $c > 0$ is an absolute constant.
\end{lemma}
\begin{proof}
	The claim follows by covering $\B_2(0,B)$ at resolution $\delta$ and applying a pigeonhole argument across the depth-$\lfloor R/2\rfloor$ subtrees on $\rge$. Noting that the number of cover elements grows as $(B/\delta)^k$, while on $\mathcal{E}^{\mathrm{reg}}_{R,\eta}$ there are at least $\exp(({\log m - \eta})R/2)$ subtrees at depth $R/2$, each containing at least one leaf. Choosing $\delta$ to balance the number of cover elements against the number of subtrees yields the stated exponential scale. The full proof, including the covering argument and verification that $d_{\mathrm{corr}}(u^*, v^*)\geq \lambda R$, is in Appendix~\ref{app:complete-proof-of-the-deep-crushing-lemma}.
    
\end{proof}

The collision pair $(u^*,v^*)$ witnesses a mismatch between intrinsic and represented distances. To convert this geometric mismatch into an approximation statement, we construct a canonical binary target that separates the two child-subtrees emanating from the lowest common ancestor of $u^*$ and $v^*$ ($\mathrm{LCA}(u,v)$). This target is Lipschitz in $d_{\mathrm{corr}}$ since the cut occurs at a macroscopic depth, but any Euclidean predictor that fits the labels on the collided pair must vary sharply over an exponentially small $\ell_2$ distance.

\begin{theorem}[Euclidean Lipschitz lower bound for a canonical hierarchical cut]\label{thm:lipschitz-capacity-barrier}
	On $\rge$, let $\phi:V_{\le R}\to \R^k$ satisfy $\phi(L_R)\subseteq \B_2(0,B)$. Let $(u^*,v^*)$ be a collided pair as in Lemma~\ref{lem:deep-crushing}, and let $a\doteq\mathrm{LCA}(u^*,v^*)$. Let $c_u,c_v$ be the two distinct children of $a$ that lie on the paths from $a$ to $u^*$ and $v^*$, respectively.
	
	Define a leaf-label $g_L:L_R\to\{-1,0,+1\}$ by
\begin{equation*}
	g_L(x)\doteq
\begin{cases}
+1 & x\in L_R^{(c_u)},\\
-1 & x\in L_R^{(c_v)},\\
0  & \text{otherwise}.
\end{cases}
\end{equation*}
Since $\mathrm{depth}(a)\le R/2$ and $w_e\ge \lambda$, any two leaves receiving different labels are at correlation distance at least $\lambda R$, hence $g_L$ is $(2/(\lambda R))$-Lipschitz on $(L_R,d_\mathrm{corr})$. Let $g:V_{\le R}\to[-1,1]$ be any $(2/(\lambda R))$-Lipschitz extension of $g_L$ to $V_{\leq R}$ (e.g.\ via a McShane extension; see Lemma~\ref{lem:mcshane-extension}).

	Then:
	\begin{enumerate}
		\item $g$ is $O(1/(\lambda R))$-Lipschitz with respect to $d_\mathrm{corr}$.
		\item Any $h:\R^k\to[-1,1]$ satisfying
		\begin{equation*}
			\abs{h(\phi(u^*))-g(u^*)}+\abs{h(\phi(v^*))-g(v^*)} \leq 1
		\end{equation*}
		must have
		\begin{equation*}
			\Lip_{\norm{\cdot}_2}(h) \geq \frac{c'}{B}\exp\left(\frac{\log m-4\eta}{2k}R\right),
		\end{equation*}
		for an absolute constant $c'>0$.
	\end{enumerate}
\end{theorem}

\begin{proof} The proof proceeds as follows:
	\paragraph{Item 1.} Across the cut $(L_R^{(c_u)},L_R^{(c_v)})$ the label changes by $2$ while $d_{\mathrm{corr}}\geq \lambda R$, so $\Lip_{d_{\mathrm{corr}}}(g_L)\leq 2/(\lambda R)$. A McShane extension preserves the Lipschitz constant, hence the same bound holds for $g$ on $V_{\le R}$. (The additional value $0$ off the two chosen subtrees can only reduce the worst-case slope.)

	\paragraph{Item 2.} Since $g(u^*)=1$ and $g(v^*)=-1$, the assumption implies
	\begin{equation*}
		\abs{h(\phi(u^*))-h(\phi(v^*))} \geq \abs{g(u^*)-g(v^*)} - \big(\abs{h(\phi(u^*))-g(u^*)}+\abs{h(\phi(v^*))-g(v^*)}\big) \geq 1.
	\end{equation*}
	Dividing by $\norm{\phi(u^*)-\phi(v^*)}_2$ and using Lemma~\ref{lem:deep-crushing} yields the claimed Lipschitz lower bound.

\end{proof}

\paragraph{Normalized distance.}
Theorem~\ref{thm:lipschitz-capacity-barrier} is an approximate lower bound driven by geometry. The target $g$ becomes increasingly smooth in the intrinsic metric as depth grows: $\mathrm{Lip}_{d_{\mathrm{corr}}}(g) = O(1/R)$. Yet any predictor on a bounded Euclidean representation that achieves constant approximation error must have Euclidean Lipschitz constant at least $\exp(\Omega(R/k))/B$. Equivalently, if we normalize distances by the tree diameter $\lambda R$ to obtain $\tilde{d}_{\mathrm{corr}} \doteq d_{\mathrm{corr}}/(\lambda R)$, then $g$ is $O(1)$-Lipschitz under $\tilde{d}_{\mathrm{corr}}$, while its realization on $\phi(L_R) \subseteq B_2(0, B)$ requires $\mathrm{Lip}_{\ell_2}(h) \geq \exp(\Omega(R/k))/B$. Allowing $B$ to grow exponentially can avoid collisions, but then standard covering/packing bounds for Lipschitz classes on $\B_2(0,B)$ reintroduce exponential statistical complexity through the domain size.

\paragraph{Statistical implication (via capacity control).} If the learner restricts itself to predictors with Euclidean Lipschitz constant $\Lip_{\norm{\cdot}_2}(h)\leq \poly(R)$ (as induced by standard norm/Lipschitz regularization), then it incurs a constant approximation error on some $O(1/R)$-Lipschitz hierarchical target. Achieving realizability forces $\Lip_{\norm{\cdot}_2}(h)\ge \exp(\Omega(R/k))$, and fat-shattering/packing lower bounds for Lipschitz classes on $\B_2(0,B)$ yield exponential sample complexity in $R$; see Appendix~\ref{app:fat-shattering-and-packing}.

\section{Low-Rank Representations}\label{sec:low-rank-representation-barrier}

A natural objection to the above argument is that geometry may be the wrong bottleneck: perhaps the learner can bypass exponential sample complexity by using compressed prediction spaces (e.g., a width-$k$ bottleneck, a linear probe on learned features, or a finite-rank kernel). Here we show that regardless of representation, any rank-$k$ prediction space can align with only $O(k)$ independent multiscale contrasts.

We analyze balanced $m$-ary trees\footnote{It is easy to extend the argument to general Galton-Watson trees with bounded offspring; we focus on deterministic trees for clarity.} with $N=m^R$ leaves and study limitations imposed by low-dimensional prediction spaces. Let $\Phi\in\R^{N\times k}$ be the feature matrix on the leaves, with rows $\phi(x_i)^\top$. Any linear readout $h(x)=w^\top\phi(x)$ induces a prediction vector on the leaves $\mathbf h=\Phi w\in\R^N$, so the set of realizable predictions is contained in the subspace $S=\mathrm{col}(\Phi)$ with $\dim(S)\le k$.

More generally, the same analysis applies whenever the learner's realizable predictions on the $N$ leaves lie in some subspace $S\subseteq\R^N$ with $\dim(S)\leq k$ (e.g., a width-$k$ bottleneck at the final layer or a finite-rank kernel restricted to the leaves).

\begin{definition}[Tree-Haar wavelets]\label{def:tree-haar-wavelets}
	For a balanced $m$-ary tree of depth $R$ with $N = m^R$ leaves, fix an orthonormal contrast basis $\{a^{(1)}, \dots a^{(m-1)}\} \subset \R^m$ orthogonal to $\mathbf{1}$. For any internal node $v \in V \setminus L_R$, let $S_v \subseteq L_R$ denote the set of leaves in the subtree rooted at $v$. Let $v_1, \dots, v_m$ be the children of $v$, and let $S_{v,i} = S_{v_i}$ denote the subset of leaves descending from the $i$-th child. Note that $S_v$ is the disjoint union $\bigsqcup_{i=1}^m S_{v,i}$.
	
	 For each internal node $v$ and $j \in [m-1]$, define the wavelet $\psi_{v,j} \in \R^N$ by
	\begin{equation*}
		\psi_{v,j}(x) = \begin{cases} 
			\frac{a^{(j)}_i}{\sqrt{\abs{S_v}/m}} & \text{if } x \in S_{v,i} \text{ (leaves under } v\text{'s } i\text{-th child)}, \\ 
			0 & \text{if } x \notin S_v.
		\end{cases}
	\end{equation*}
	Let $\mathcal W = \{\psi_{v,j} \mid v\text{ internal}, j \in [m-1]\}$.
\end{definition}

Each local refinement label (child $c^*$ versus its siblings' at node $v$) is a constant shift of a vector in $\operatorname{span}\{\psi_{v,j}\}_{j=1}^{m-1}$; thus rank limits directly constrain how many refinement problems can be represented simultaneously. The collection $\mathcal W$ forms an orthonormal basis of $\mathbf{1}^\perp\subset\R^N$ and satisfies $\abs{\mathcal W}=N-1$ (verified in Appendix~\ref{app:orthonormality-details}).

We now show that no $k$-dimensional subspace can have substantial correlation with more than a $k/(N-1)$ fraction of these $N-1$ orthonormal contrasts on average. The proof is a one-line trace identity once we use that $\mathcal W$ is an orthonormal basis of $\mathbf{1}^\perp$.

\begin{theorem}[Average alignment bound]\label{thm:alignment-bound}
Let $S\subseteq\R^N$ be any subspace with $\dim(S)\le k$. Then
\begin{equation*}
	\frac{1}{\abs{\mathcal W}}\sum_{\psi\in\mathcal W}\norm{P_S\psi}_2^2 \leq \frac{k}{N-1}.
\end{equation*}
\end{theorem}
\begin{proof}
$\mathcal W$ is an orthonormal basis of $\mathbf{1}^\perp$, hence $\sum_{\psi\in\mathcal W}\psi\psi^\top=P_{\mathbf{1}^\perp}$. Therefore
\begin{equation*}
	\sum_{\psi\in\mathcal W}\norm{P_S\psi}_2^2
	= \Tr\left[P_S\sum_{\psi\in\mathcal W}\psi\psi^\top\right]
	= \Tr[P_S P_{\mathbf{1}^\perp}]
	\leq \Tr[P_S]
	= \dim(S)
	\leq k.
\end{equation*}
Dividing by $\abs{\mathcal W}=N-1$ completes the proof.
\end{proof}

The average bound immediately implies that most hierarchical contrasts are nearly orthogonal to any rank-$k$ prediction space, yielding lower bounds on approximation error.

\begin{corollary}[Consequences for rank-$k$ prediction spaces]\label{cor:bottleneck-representation}
Let $S\subseteq\R^N$ be any subspace with $\dim(S)\le k$ (in particular $S=\mathrm{col}(\Phi)$ for $k$-dimensional features and a linear readout).
\begin{itemize}
\item \emph{Average approximation error.} The mean squared error of the best approximation of each wavelet by $S$ satisfies
\begin{equation*}
	\frac{1}{N-1}\sum_{\psi\in\mathcal W}\min_{s\in S}\norm{\psi-s}_2^2 \geq 1-\frac{k}{N-1}.
\end{equation*}

Equivalently, when $S=\mathrm{col}(\Phi)$ this is $\min_{w\in\R^k}\norm{\psi-\Phi w}_2^2$.

\item \emph{Fraction of tasks approximated.} If $\min_{s\in S}\norm{\psi-s}_2^2\leq \varepsilon$ holds for at least a $(1-\eta)$ fraction of $\psi\in\mathcal W$, then
	\begin{equation*}
		k \geq (1-\eta)(1-\varepsilon)(N-1).
	\end{equation*}
\end{itemize}
Here $\norm{\cdot}_2$ is the ambient norm on $\R^N$ and each $\psi\in\mathcal W$ is unit-normalized, so $\varepsilon$ is an absolute squared-error scale (not per-leaf MSE).
\end{corollary}

\begin{proof}
For a unit vector $\psi$, the best approximation error onto $S$ is
\begin{equation*}
	\min_{s\in S}\norm{\psi-s}_2^2 = 1-\norm{P_S\psi}_2^2.
\end{equation*}
Averaging over $\mathcal W$ and applying Theorem~\ref{thm:alignment-bound} yields the first claim.

For the second claim, let $G\subseteq\mathcal W$ be the set of wavelets satisfying $\min_{s\in S}\norm{\psi-s}_2^2\leq\varepsilon$, i.e., $\norm{P_S\psi}_2^2\geq 1-\varepsilon$. If $\abs{G}\geq (1-\eta)(N-1)$, then
\begin{equation*}
	k \geq \sum_{\psi\in\mathcal W}\norm{P_S\psi}_2^2 \geq \sum_{\psi\in G}\norm{P_S\psi}_2^2 \geq (1-\eta)(N-1)(1-\varepsilon).
\end{equation*}
\end{proof}

Although Corollary~\ref{cor:bottleneck-representation} is stated in terms of a linear subspace $S \subset \R^N$, the conclusion applies to any architecture whose predictions depend on the data only through a rank-$k$ representation. Concretely, if a deep network contains an intermediate linear layer $\Phi \in \R^{N\times k}$, then all subsequent (possibly nonlinear) computations see the data only through $\operatorname{col}(\Phi)$, and the projected wavelets $P_S \psi$ with $S = \operatorname{col}(\Phi)$ obey the alignment bound. This result is a representation-agnostic limitation of rank-$k$ bottlenecks, while Section~\ref{sec:euclidean-barrier} and Section~\ref{sec:hyperbolic-achievability-and-optimality} compare Euclidean versus hyperbolic geometries when the prediction space is not rank-constrained but is subject to standard Lipschitz or norm control.

The conclusion is that a rank-$k$ prediction space cannot simultaneously represent more than $O(k)$ independent hierarchical contrasts. This does not contradict the existence of efficient learners for specific structured task families (such as the single-path family outlined in Section~\ref{subsec:learning-protocol}), but it rules out any approach that aims to fit a large fraction of canonical multiscale contrasts using a low-rank bottleneck.

\section{Hyperbolic Achievability and Optimality}\label{sec:hyperbolic-achievability-and-optimality}

In this section we show that, for the canonical protocol of Section~\ref{subsec:learning-protocol}, hyperbolic representations achieve the minimax sample complexity. The argument separates geometry from information: (i) a constant-distortion embedding into hyperbolic space makes each local refinement decision realizable with constant geometric control (constant margin / $O(1)$-Lipschitz readout), and (ii) the remaining $\Theta(mR\log m)$ sample requirement is information-theoretic and holds under any representation.

We begin by recalling that trees admit near-isometric embeddings into hyperbolic space once curvature is chosen to accommodate the branching factor. This resolves the volume-growth mismatch that drives the Euclidean lower bound.

\begin{theorem}[Hyperbolic embedding existence]\label{thm:hyperbolic-embedding-existence}
	For any rooted tree with maximum degree $\Delta$ and edge weights satisfying $w_e \geq \lambda$, and any $\err > 0$, there exists a $(1+\err)$-bi-Lipschitz embedding into $\Hkk$ (hyperbolic space of curvature $-\kappa$) for $k \geq 2$ when
	\begin{equation*}
		\sqrt{\kappa} \geq \frac{C_k \log \Delta}{\lambda \err}.
	\end{equation*}
\end{theorem}
\begin{proof}
	This follows from an extension of \citet{sarkar_low_2011}, generalized to $k>2$ in \citet{de_sa_representation_2018}. The full algorithm and proof is given in Appendix~\ref{app:generalized-sarkar-construction}.
\end{proof}

\paragraph{Curvature scaling.}
The condition $\sqrt{\kappa}\gtrsim (\log \Delta)/(\lambda\varepsilon)$ matches the intuition that curvature must scale with branching to provide enough angular room for child subtrees. In Appendix~\ref{app:gw-packing-converse} we show a converse on $\rge$: any bi-Lipschitz embedding of a Galton-Watson tree into $\Hkk$ requires $\sqrt{\kappa}\ge \Omega(\log m)$ (up to constants depending on distortion and dimension). Thus, the dependence on $\log \Delta$ is unavoidable in this regime and the generalized Sarkar construction is optimal with respect to the geometry of the target space.

A key consequence of the generalized Sarkar construction is uniform cone separation: for any internal node $a$, the $m$ child subtrees of $a$ occupy disjoint cones in $\Hkk$ with constant angular separation, yielding a constant-margin separating geodesic hypersurface between any chosen child and its siblings (formalized in Appendix~\ref{app:hyperbolic-geometric-separation}). This realizes each refinement 
label with O(1) geometric complexity.

\begin{theorem}[Achievability for path identification]\label{thm:sample-complexity-upper-bound}
Consider the canonical protocol of Section~\ref{subsec:learning-protocol}. Each $\theta\in[m]^R$ induces a hypothesis $h_\theta(\phi(v),i)\doteq \mathbb{I}\{\mathrm{ch}_i(v)=\theta_i\}$, so $\abs{\mathcal H_{\mathrm path}}=m^R$.
\begin{enumerate}
\item \emph{Geometric realizability.} Under a constant-distortion hyperbolic embedding $\phi$, each depth-$i$ refinement label is realizable by a function on $\phi(V)$ with Lipschitz constant $O(1)$ (equivalently, by a constant-margin geodesic halfspace classifier).
\item \emph{Statistical achievability.} There exists a (polynomial-time) estimator $\hat\theta$ such that, for any $\delta\in(0,1)$, with
\begin{equation*}
	n = O\left(\frac{mR\log(mR/\delta)}{(1-2\rho)^2\acc^2}\right)
\end{equation*}
samples, the induced predictor achieves error at most $\acc$ with probability at least $1-\delta$ for fixed $\rho\in(0,1/2)$.
\end{enumerate}
\end{theorem}
\begin{proof}
	\paragraph{Realizability.} Fix depth $i$ and the on-path prefix $u_{i-1}(\theta)$. The Bayes label is membership in the distinguished child-subtree $L_R^{(c^*)}$. By cone separation (Lemma~\ref{lem:cone-separation}), there exists a totally geodesic hypersurface separating $L_R^{(c^*)}$ from sibling subtrees with margin $\gamma=\gamma(k,\varepsilon,\kappa,\tau)>0$ as in Lemma~\ref{lem:cone-separation}. Lemma~\ref{lem:lipschitz-margin-classifier} then yields an explicit $(1/\gamma)$-Lipschitz classifier realizing this task. Full details appear in Appendix~\ref{app:hyperbolic-geometric-separation}.
	
	\paragraph{Sample complexity.} We analyze a depth-wise estimator that mirrors the structure of the protocol. For each depth $i$, use the subset of samples with $I=i$. Under the protocol, conditional on $I=i$, the child index $C=\mathrm{ch}_i(V)$ is (approximately) uniform on $[m]$ on $\rge$, and $Y$ is a $\mathrm{BSC}(\rho)$ observation of $\mathbb{I}\{C=\theta_i\}$. Hence,
	\begin{equation*}
		\mathbb{E}[Y\mid I=i, C=c] = \begin{cases}
			1-\rho & c=\theta_i,\\
			\rho & c\neq \theta_i,
		\end{cases}
	\end{equation*}
	so the distinguished child has a constant mean gap of $(1-2\rho)$ over the others. Let $\hat\theta_i$ be the child index maximizing the empirical mean of $Y$ among the $m$ groups (ties broken arbitrarily). A Hoeffding bound plus a union bound over the $m$ groups implies that
\begin{equation*}
  n_i=\Theta\left(\frac{m\log(m/\delta_i)}{(1-2\rho)^2\acc^2}\right)
\end{equation*}
samples at depth $i$ suffice to ensure the empirical maximizer $\hat\theta_i$ achieves conditional error at most $\acc$ with probability at least $1-\delta_i$. Taking $\delta_i=\delta/R$ and using that $n_i\approx n/R$ with high probability (since $I\sim\mathrm{Unif}([R])$) yields the stated overall sample complexity
\begin{equation*}
  n = O\!\left(\frac{mR\log(mR/\delta)}{(1-2\rho)^2\,\acc^2}\right).
\end{equation*}

\end{proof}

It remains to show that the $\Theta(mR\log m)$ scaling is unavoidable and not an artifact of the analysis or of hyperbolic geometry. The next theorem establishes a representation-agnostic information-theoretic lower bound under an oracle model that is strictly more informative than the protocol of Section~\ref{subsec:learning-protocol}; the same lower bound therefore applies a fortiori to our setting.

\begin{theorem}[Information-theoretic lower bound]\label{thm:information-theoretic-lower-bound}
Fix $m\ge 2$ and $\rho\in(0,1/2)$. Any estimator for the hierarchical path identification task (Section~\ref{subsec:learning-protocol}) achieving constant excess risk requires
\begin{equation*}
  n \geq c \frac{m}{\beta_\rho} R\log m,
\end{equation*}
for a universal constant $c>0$, where $\beta_\rho \doteq D_{\mathrm{KL}}(\mathrm{Bern}(1-\rho)\ \|\ \mathrm{Bern}(\rho)) = (1-2\rho)\log\frac{1-\rho}{\rho}$. For $\rho$ bounded away from $\{0,1/2\}$, $\beta_\rho=\Theta((1-2\rho)^2)$, so the bound is equivalently $n=\Omega\big(mR\log m/(1-2\rho)^2\big)$ up to absolute constants.
\end{theorem}

\begin{proof}
    Identifying $\theta\in[m]^R$ requires $\Theta(R\log m)$ bits. In the oracle model, a sample reveals the depth $I$ and a uniformly random child index $C\in[m]$, and the label $Y$ differs across parameters only on the events $\{C=\theta_I\}$ or $\{C=\theta'_I\}$, which have probability $2/m$. Consequently, the per-sample KL between distinct parameters is $O(1/m)$, so Fano's inequality forces $n=\Omega(mR\log m)$.
    
    Formally: we construct a Varshamov-Gilbert packing $\Theta_0$ with metric entropy $\log \abs{\Theta_0} \asymp R \log m$. Due to the sparsity of informative leaves, the KL divergence between hypotheses scales as $O(n/m)$. From Fano's inequality, we require $n/m \gtrsim R \log m$ to distinguish parameters. A full derivation is given in Appendix~\ref{app:fano-details}.
\end{proof}

Together, Theorems~\ref{thm:sample-complexity-upper-bound} and \ref{thm:information-theoretic-lower-bound} show that the canonical protocol has minimax sample complexity $\Theta(mR\log m)$, and that hyperbolic representations achieve this rate with constant geometric control. Combined with the Euclidean approximation lower bound in Section~\ref{sec:euclidean-barrier}, this yields the claimed separation under bounded-radius Euclidean representations with polynomial complexity control.

\section{Discussion}\label{sec:conc}

\subsection{Related Work}

\paragraph{Hyperbolic representations.} Hyperbolic embeddings were introduced for taxonomies with strong performance in low dimension and were followed by architectural extensions to hyperbolic neural networks and hyperbolic GNNs for hierarchical and scale-free graphs \citep{nickel_poincare_2017,nickel_learning_2018,ganea_hyperbolic_2018,chami_hyperbolic_2019}. Low-distortion embeddings of trees into hyperbolic space are well understood, notably, constructions in the Poincar\'{e} plane \citep{sarkar_low_2011} and precision-dimension tradeoffs \citep{de_sa_representation_2018}. Our results are orthogonal to these distortion guarantees: by quantifying realizability under Lipschitz/norm control, we obtain a statistical separation in sample complexity between bounded-radius Euclidean and hyperbolic representations.

\paragraph{Metric embeddings beyond distortion.} Classical metric embedding theory concerns multiplicative distortion, from Bourgain’s $O(\log n)$ bound for general metrics to sharper bounds for trees, which embed with small distortion in high dimension but can require polynomial distortion in fixed-dimension Euclidean space \citep{Matousek1999}. We instead quantify the statistical implications of embedding: bounded-radius Euclidean embeddings exhibit a multi-scale collision phenomenon that we convert into an exponential Lipschitz requirement and hence exponential sample complexity under capacity control. In this sense, representation geometry imposes a statistical penalty even when coarse distortion can be favorable.

\paragraph{Capacity, Lipschitz learning, and bottlenecks.} Generalization for Lipschitz predictors is characterized by fat-shattering and packing/covering dimensions in both Euclidean and general metric spaces \citep{10.1145/263867.263927,von_luxburg_classification_2004}. We use this machinery to show that Euclidean collapse induces exponential Lipschitz constants and hence exponential sample complexity for hierarchical refinement. We also prove a representation-agnostic expressivity bound: any rank-$k$ prediction space (including bottlenecks and finite-rank kernels) captures only $O(k)$ out of $m^R-1$ canonical hierarchical contrasts. This connects to multiscale bases on trees and graphs (e.g., diffusion wavelets) but here serves as a lower bound on the learnable signal, independent of embedding geometry.

\paragraph{Representation vs. learning.} Empirically, hyperbolic representations often outperform Euclidean ones in low dimension for hierarchical data, while at higher dimension Euclidean embeddings can match or exceed hyperbolic reconstruction quality. Recent reproductions report that Euclidean embeddings can close the gap around $\gtrsim 50$ dimensions \citep{bansal-benton-2021-comparing}. Furthermore, recent evidence indicates that hyperbolic neural networks do not achieve theoretical optimal embeddings in practice \citep{tan2024hyperbolicneuralnetworkseffective}. Together, these observations suggest that embedding distortion or reconstruction quality alone is insufficient to predict statistical efficiency. Our bounds make this explicit: Euclidean representations incur a penalty scaling as $\exp(\Omega(R/k))$ under bounded radius and standard regularization, while hyperbolic representations achieve minimax-optimal $\Theta(mR\log m)$ rates.

\subsection{Summary}
We show that representation geometry governs sample complexity for hierarchical learning. For the canonical refinement task on depth-$R$ $m$-ary (or bounded Galton-Watson) trees, bounded-radius Euclidean representations require exponential smoothness and sample complexity, whereas constant-distortion hyperbolic representations admit $O(1)$-Lipschitz realizers and achieve the minimax rate $\Theta(mR\log m)$. This identifies hyperbolic space as a statistically efficient geometric prior for hierarchical targets and clarifies when and why Euclidean models require higher dimension to avoid exponential penalties.

\acks{DR acknowledges support from VESSL AI.}

\bibliography{hyperbolic_embeddings.bib}

\appendix

\section{Probabilistic Foundations}\label{app:probabilistic-and-statistical-mechanical-foundations}

\subsection{Probability of the Regular Growth Event}\label{app:probability-of-the-regular-growth-event}

\begin{definition}[Regular growth event]\label{def:galton-watson-tree-with-regular-growth-event}
	Let $T$ be generated by a Galton-Watson branching process with offspring distribution $\xi$ satisfying $\mathbb{E}[\xi] = m > 1$ and $\mathbb{E}[\xi^2] < \infty$. For $\eta \in (0, \log m)$, we define two conditions:
	\begin{enumerate}
		\item \emph{Global size.} $\mathcal{E}_{R,\eta,\mathrm{size}}^{\mathrm{reg}}\doteq\bigcap_{t\in\{\lfloor R/2\rfloor,R\}}\left\{\exp((\log m-\eta)t)\leq |L_t|\le \exp((\log m+\eta)t)\right\}$.
		\item \emph{Subtree dominance.} $\mathcal{E}_{R,\eta,{\mathrm max}}^{\mathrm{reg}} \doteq \big\{ \max_{w \in L_{\lfloor R/2 \rfloor}} \abs{L_R^{(w)}} \leq \exp((\log m + \eta)R/2) \big\}$.
	\end{enumerate}
	We define the regular growth event as their intersection: $\rge \doteq \mathcal{E}_{R,\eta,{\mathrm size}}^{\mathrm{reg}} \cap \mathcal{E}_{R,\eta,{\mathrm max}}^{\mathrm{reg}}$. We shall also always take $0 < \eta < \tfrac{1}{4}\log m$.
\end{definition}

The above two-sided bound ensures (i) exponentially many leaves, and (ii) that no single mid-level subtree dominates. Both conditions are necessary for the multi-scale pigeonhole argument in Lemma~\ref{lem:deep-crushing}.

We establish that the regular growth event $\rge$ holds with high probability under standard Galton-Watson assumptions.
\begin{lemma}[Typicality of regular growth]\label{lem:typicality-of-regular-growth}
	Let $T$ be a Galton-Watson tree with offspring distribution $\xi$. Assume $\mathbb{E}[\xi] = m > 1$ and that $\xi$ has bounded support ($\xi \le K$ almost surely). For any fixed $\eta > 0$,
	\begin{equation*}
		\lim_{R \to \infty} \Pr\left[\rge \big\vert \mathsf{Surv}\right] = 1.
	\end{equation*}
\end{lemma}

\begin{proof}
	We show that $\Pr[{\rge}^c \mid \mathsf{Surv}] \to 0$ as $R \to \infty$.
	
	\paragraph{Global size ($\mathcal{E}_{R,\eta,{\mathrm size}}^{\mathrm reg}$).} Kesten-Stigum results \citep{kesten_additional_1966} imply $\abs{L_t}/m^t \to W$ almost surely on survival (and in $L^1$ under $\mathbb{E}[\xi\log\xi]<\infty$). In particular, for any fixed $\eta>0$, $\Pr[\exp((\log m-\eta)t)\leq~\abs{L_t}\leq \exp((\log m+\eta)t)\mid \mathsf{Surv}] \to 1$ for each fixed $t\in\{\lfloor R/2\rfloor,R\}$.

	\paragraph{Subtree dominance ($\mathcal{E}_{R,\eta,{\mathrm max}}^{\mathrm reg}$).}
	Condition on the nodes at depth $R/2$, denoted $L_{R/2}$. Let $K = \abs{L_{R/2}}$. The sizes of subtrees rooted at this depth, $\{Z_1, \dots, Z_K\}$, are independent. We require $\max_i Z_i \leq m^{R/2} e^{\eta R/2}$.
	
	Since $\xi$ has bounded support, the branching process admits exponential moment bounds. In particular, for any fixed $\eta>0$ there exists $c_\eta>0$ such that for a subtree of depth $R/2$,
	\begin{equation*}
		\Pr\left[ Z_1 \geq \exp((\log m+\eta)R/2)\right] \leq \exp(-c_\eta R),
	\end{equation*}
	see, e.g., Cram\'er-Chernoff bounds for Galton-Watson processes with bounded offspring.
	Taking a union bound over $K=\abs{L_{R/2}} \le	 \exp((\log m+\eta)R/2)$ subtrees on $\mathcal{E}^{\mathrm reg}_{R,\eta,\mathrm{size}}$ yields
	\begin{equation*}
		\Pr[\mathcal{E}^{\mathrm reg}_{R,\eta,\max}{}^c \mid \mathsf{Surv}] \to 0.
	\end{equation*}
\end{proof}

\subsection{Ising Model and Correlation Decay}\label{app:ising-model}

The main text treats $d_{\mathrm{corr}}$ as an intrinsic hierarchical similarity metric. This subsection provides one concrete generative instantiation in which $d_{\mathrm{corr}}$ arises canonically: on a tree-structured graphical model, correlations factor along paths, so negative log-correlation is exactly an additive weighted tree distance. This makes the Lipschitz assumption in $d_{\mathrm{corr}}$ interpretable as smoothness with respect to the underlying generative statistics.

\begin{definition}[Ferromagnetic Ising model on a finite graph]\label{def:ferromagnetic-ising-model-on-a-finite-graph}
	Let $G = (V, E)$ be a finite, simple, connected graph. Fix edge couplings $\{J_e\}_{e \in E}$ with $J_e \geq 0$ and zero external field. The Ising measure is
	\begin{equation*}
		\mathbb{P}(x) \propto \exp\left(\sum_{e = \{u,v\} \in E} J_e x_u x_v\right), \quad x \in \{-1, +1\}^V.
	\end{equation*}
	Write $\langle \cdot \rangle_G$ for expectation under $\mathbb{P}$.
\end{definition}

\paragraph{Generality of the Ising model.} We focus on the Ising model as the canonical example of a discrete Markov Random Field (MRF). However, the connection between correlation decay and weighted tree metrics is a general feature of graphical models on trees. For instance, in Gaussian Graphical Models on trees, the correlation decays exponentially with path length, inducing a similar geometry. We analyze the Ising case because its discrete nature presents the sharpest information-theoretic lower bounds (as formalized in Appendix~\ref{app:fano-details}), but the geometric intuitions extend to general spin systems and latent tree models.

\begin{definition}[Correlation dissimilarity]
	Let $(X_v)_{v\in V}$ denote the random spin configuration under $\mathbb P$. For $u \neq v$, define
	\begin{equation*}
		d_{\mathrm{corr}}(u,v) \doteq -\log \langle X_u X_v \rangle_G,
	\end{equation*}
	with the convention $d_{\mathrm{corr}}(u,v) = +\infty$ if $\langle X_u X_v \rangle_G = 0$. 
\end{definition}

\begin{lemma}[Tree factorization of correlations]\label{lem:tree-factorization-appendix}
	Let $G = (V, E)$ be a tree with ferromagnetic Ising couplings $\{J_e\}_{e \in E}$ and zero external field. Then for any $u \neq v$,
	\begin{equation*}
		\langle X_u X_v \rangle_G = \prod_{e \in \mathrm{path}(u,v)} \tanh(J_e).
	\end{equation*}
	Consequently, $d_{\mathrm{corr}}(u,v) = \sum_{e \in \mathrm{path}(u,v)} w_e$ where $w_e \doteq -\log \tanh(J_e)$.
\end{lemma}

\begin{proof}
	The zero-field ferromagnetic Ising model on a tree defines a Markov random field with conditional independence along separators. We proceed by induction.

	\paragraph{Base case (single edge).} For neighbors $u, v$ connected by edge $e$ with coupling $J_e$:
	\begin{equation*}
		\langle X_u X_v \rangle_G = \frac{\sum_{x_u, x_v} x_u x_v \exp(J_e x_u x_v)}{\sum_{x_u, x_v} \exp(J_e x_u x_v)} = \frac{2(\exp(J_e) - \exp(-J_e))}{2(\exp(J_e) + \exp(-J_e))} = \tanh J_e.
	\end{equation*}

	\paragraph{Inductive step.} Let $w$ be on the path between $u$ and $v$. By conditional independence:
	\begin{equation*}
		\langle X_u X_v \rangle_G = \mathbb{E}_w[\mathbb{E}[X_u \mid X_w] \mathbb{E}[X_v \mid X_w]] = \mathbb{E}_w[X_w^2] \tanh J_{uw} \tanh J_{wv} = \tanh J_{uw} \langle X_w X_v \rangle_G.
	\end{equation*}
	Iterating along the path gives the product formula.
\end{proof}

\begin{corollary}[homogeneous tree regime]
If all couplings are equal $J_e\equiv J$, then $w_e\equiv \lambda$ with $\lambda\doteq -\log\tanh(J)$, and hence
\begin{equation*}
	d_{\mathrm corr}(u,v)=\lambda\abs{\mathrm{path}(u,v)}.
\end{equation*}
\end{corollary}

\section{Euclidean Representation Lower Bound Technical Details}\label{app:euclidean-barrier-technical-details}

\subsection{Complete Proof of the Volumetric Collapse Lemma}\label{app:complete-proof-of-the-deep-crushing-lemma}

We provide the full details of the covering argument in Lemma~\ref{lem:deep-crushing}.

\begin{proof}
	
	\paragraph{Setup.} Partition $L_R$ by ancestry at depth $R/2$. For each node $w$ at depth $R/2$, let $L_R^{(w)} \subseteq L_R$ denote the leaves in the subtree rooted at $w$. These sets partition $L_R$.

	\paragraph{Covering.} Since $\phi(L_R)\subseteq \B_2(0,B)$, it suffices to cover $\B_2(0,B)$ by sets of $\ell_2$-diameter at most $\delta$. Cover the ball $\B_2(0, B)$ with a grid of $\ell_\infty$-cubes of side length $\delta/\sqrt{k}$. Each cube has $\ell_2$-diameter at most $\delta$. The number of cubes needed to cover $\B_2(0, B)$ is at most
	\begin{equation*}
		\left\lceil \frac{2B}{\delta/\sqrt{k}} \right\rceil^k \leq \left(\frac{2\sqrt{k} B}{\delta} + 1\right)^k \leq \left(\frac{cB}{\delta}\right)^k
	\end{equation*}
	for a constant $c = c(k)$ and $\delta$ sufficiently small relative to $B$.

	\paragraph{Pigeonhole.} We claim that if
	\begin{equation}\label{eq:pigeonhole-condition}
		\abs{L_R} > \left(\frac{cB}{\delta}\right)^k \max_w \abs{L_R^{(w)}},
	\end{equation}
	then some cube contains leaves from two distinct depth-$R/2$ subtrees.

	Suppose for contradiction that every cube contains leaves from at most one depth-$R/2$ subtree. Then
	\begin{equation*}
		\abs{L_R} = \sum_{\text{cubes } C} \abs{C \cap \phi(L_R)} \leq \sum_{\text{cubes } C} \max_w \abs{L_R^{(w)}} \leq \#\text{cubes} \cdot \max_w \abs{L_R^{(w)}},
	\end{equation*}
	contradicting equation~\eqref{eq:pigeonhole-condition}.

	\paragraph{Applying the growth bounds.} On $\rge$:
	\begin{itemize}
		\item $\abs{L_R} \geq \exp((\log m - \eta)R)$
		\item $\max_w \abs{L_R^{(w)}} \leq \exp((\log m + \eta)R/2)$ (each subtree has depth $R/2$)
	\end{itemize}

	Set $\delta = cB \exp\left(-\frac{\log m - 4\eta}{2k}R\right)$. Then condition \eqref{eq:pigeonhole-condition} becomes:
	\begin{align*}
		\exp((\log m - \eta)R) &> \left(\frac{cB}{cB \exp(-(\log m - 4\eta)R/(2k))}\right)^k \exp((\log m + \eta)R/2) \\
		&= \exp\left(\frac{(\log m - 4\eta)R}{2}\right) \exp\left(\frac{(\log m + \eta)R}{2}\right) \\
		&= \exp\left(\left(\log m - \tfrac{3}{2}\eta\right)R\right).
	\end{align*}
	This holds for every $\eta>0$ (and all $R$), since $\log m-\eta > \log m-\tfrac{3}{2}\eta$.

	\paragraph{Conclusion.} Some cube of diameter $\delta$ contains $u \in L_R^{(w)}$ and $v \in L_R^{(w')}$ with $w \neq w'$. Thus:
	\begin{itemize}
		\item $\norm{\phi(u) - \phi(v)}_2 \leq \delta = cB \exp\left(-\frac{\log m - 4\eta}{2k}R\right)$
		\item The lowest common ancestor of $u,v$ has depth $\le R/2$, so the unweighted path length satisfies $\abs{\mathrm{path}(u,v)}\geq R$. Since $w_e\geq \lambda$ for all edges, it follows that
\begin{equation*}
	d_{\mathrm{corr}}(u,v)=\sum_{e\in\mathrm{path}(u,v)} w_e \geq \lambda\abs{\mathrm{path}(u,v)} \geq \lambda R.
\end{equation*}

	\end{itemize}
\end{proof}

\subsection{Lipschitz Extension}\label{app:lipschitz-extension}

\begin{lemma}[McShane extension]\label{lem:mcshane-extension}
	Let $(\mathcal{X}, d_\mathcal{X})$ be a metric space and $A \subseteq \mathcal{X}$. If $g_A : A \to \R$ is $L$-Lipschitz, then it extends to an $L$-Lipschitz function $g : \mathcal{X} \to \R$.
\end{lemma}

\begin{proof}
	This is due to \citet{McSHANE1934ExtensionOR}. The extension is given by
	\begin{equation*}
		g(x) \doteq \inf_{a \in A} \{g_A(a) + L d_\mathcal{X}(x, a)\}.
	\end{equation*}
	For any $x, y \in \mathcal{X}$ and any $a \in A$:
	\begin{equation*}
		g(x) \leq g_A(a) + L d_\mathcal{X}(x, a) \leq g_A(a) + L d_\mathcal{X}(x, y) + L d_\mathcal{X}(y, a).
	\end{equation*}
	Taking infimum over $a$ on the right: $g(x) \leq g(y) + L d_\mathcal{X}(x, y)$. By symmetry, $\abs{g(x) - g(y)} \leq L d_{\mathcal X}(x, y)$.

	To obtain range $[-1, 1]$, compose with the clipping map $t \mapsto \min(1, \max(-1, t))$, which is 1-Lipschitz.
\end{proof}

\subsection{Fat-Shattering and Packing}\label{app:fat-shattering-and-packing}

\begin{lemma}[Fat-shattering lower bound via packing]\label{lem:fat-shattering-lower-bound-via-packing}
	Let $X \subseteq \Rk$ and let $\mathcal{H}_\Lambda = \{h : X \to [-1, +1] \mid \Lip_{\norm{\cdot}_2}(h) \leq \Lambda\}$. For any $\acc \in (0, 1)$:
	\begin{equation*}
		\fat_\acc(\mathcal{H}_\Lambda) \geq \mathsf{Pack}_X(2\acc/\Lambda).
	\end{equation*}
\end{lemma}

\begin{proof}
	Let $S = \{x_1, \dots, x_N\} \subseteq X$ be a $(2\acc/\Lambda)$-separated set with $N = \mathsf{Pack}_X(2\acc/\Lambda)$. We show $S$ is $\acc$-shattered by $\mathcal{H}_\Lambda$ with witnesses $r_i = 0$.

	For any sign pattern $\sigma \in \{-1, +1\}^N$, define values $y_i = \sigma_i \acc$. For distinct $i, j$:
	\begin{equation*}
		\frac{\abs{y_i - y_j}}{\norm{x_i - x_j}_2} \leq \frac{2\acc}{2\acc/\Lambda} = \Lambda.
	\end{equation*}
	Thus the assignment is $\Lambda$-Lipschitz on $S$. By McShane extension (Lemma~\ref{lem:mcshane-extension}), extend to $h_\sigma : X \to [-1, 1]$ with $\Lip(h_\sigma) \leq \Lambda$.

	Then for each $i$, $\sigma_i(h_\sigma(x_i)-r_i)=\sigma_i(\sigma_i\acc)=\acc$. Since this holds for all $2^N$ patterns, $\fat_\acc(\mathcal{H}_\Lambda) \geq N$.
\end{proof}

\begin{proposition}[Sample complexity lower bound]\label{prop:fat-shattering-lower-bound}
	Let $\mathcal{H}$ be a class of real-valued functions. There exist absolute constants $c, c' > 0$ such that for any $\acc \in (0, c)$, any learning algorithm with expected error at most $\acc$ requires
	\begin{equation*}
		n \geq c' \cdot \fat_{c\acc}(\mathcal{H})
	\end{equation*}
	samples.
\end{proposition}
\begin{proof}
	This is a restatement of the lower bound for scale-sensitive function classes \citep[Theorem 19.5]{anthony_bartlett_1999}.
\end{proof}

\begin{corollary}[Euclidean sample complexity]\label{cor:euclidean-complexity}
	Let $\phi$ be a Euclidean representation containing a $\delta$-packing of size $M$ on the leaves. For the unregularized Lipschitz class $\mathcal{H}_\Lambda$ with with $\Lambda \ge 2\varepsilon/\delta$ (here $\varepsilon$ denotes the fat-shattering scale corresponding to accuracy), the sample complexity is lower bounded by:
	\begin{equation*}
		n = \Omega(M).
	\end{equation*}
\end{corollary}
\begin{proof}
	This follows immediately from the relationship between the fat-shattering dimension and the packing number established in the previous lemma:
	\begin{equation*}
		\fat_\varepsilon(\mathcal{H}_\Lambda) \geq \mathsf{Pack}_X(2\varepsilon/\Lambda).
	\end{equation*}
	By assumption, the representation contains a $\delta$-packing of size $M$, which implies $\mathsf{Pack}_X(\delta) \ge M$. Setting the resolution parameters such that $2\varepsilon/\Lambda \le \delta$ (equivalently $\Lambda \ge 2\varepsilon/\delta$), we obtain:
	\begin{equation*}
		\fat_\varepsilon(\mathcal{H}_\Lambda) \geq M.
	\end{equation*}
	Applying Proposition~\ref{prop:fat-shattering-lower-bound} with the appropriate constants yields $n = \Omega(M)$.
\end{proof}

In the main text, we apply this corollary with the exponentially small separation scale $\delta$ forced by Lemma~\ref{lem:deep-crushing}, which yields exponential sample complexity when realizability requires $\Lambda$ to be exponentially large.

\section{Hyperbolic Embedding Construction and Analysis}\label{app:hyperbolic-embedding-construction-and-analysis}

Here we record the hyperbolic embedding construction (generalized Sarkar) and the two geometric facts used in the main text: (i) existence of a constant-distortion embedding under curvature scaling, and (ii) uniform cone/margin separation between sibling subtrees, which yields realizability by constant-Lipschitz readouts.

\subsection{Generalized Sarkar Construction}\label{app:generalized-sarkar-construction}

\RestyleAlgo{ruled}
\begin{algorithm}[t]
\DontPrintSemicolon
\caption{Generalized Sarkar Construction (adapted from \citet{de_sa_representation_2018})}
\label{alg:generalized-sarkar-construction}
\SetKwFunction{MobiusToOrigin}{MobiusToOrigin}
\SetKwFunction{MobiusBack}{MobiusBack}
\SetKwFunction{SphericalCode}{SphericalCode}
\SetKwFunction{Align}{Align}
\SetKwInOut{KwIn}{Input}\SetKwInOut{KwOut}{Output}
\SetKw{Continue}{continue}
\SetKw{Return}{return}

\KwIn{Rooted tree $T=(V,E)$ with parent map $\pi(\cdot)$; dimension $k\geq 2$; edge scale $\tau>0$; curvature of space $-\kappa$.}
\KwOut{Embedding $f:V\to \mathbb{B}^k=\{x\in\mathbb{R}^k:\norm{x}<1\}$ (Poincar\'e ball model).}

\BlankLine
$\rho \leftarrow \tanh((\tau \sqrt{\kappa})/2)$ \tcp*[r]{Euclidean radius for hyperbolic step $\tau$}
$r_0 \leftarrow \mathrm{root}(T)$;\quad $f(r_0)\leftarrow 0$;

\BlankLine
\ForEach{$a\in V$ in any order that visits parents before children}{
  $C(a)\leftarrow\{c:\pi(c)=a\}$;\\
  \If{$C(a)=\emptyset$}{\Continue}
  \eIf{$a=r_0$}{
    Choose $\abs{C(a)}$ unit vectors $u_i\in\mathbb{S}^{k-1}$ from a spherical code;\\
    \ForEach{$c_i\in C(a)$}{ $f(c_i)\leftarrow \rho u_i$ }
  }{
    $b\leftarrow \pi(a)$;
    $\phi_a \leftarrow \MobiusToOrigin(f(a))$; \tcp*[r]{isometry with $\phi_a(f(a))=0$}
    $z \leftarrow \phi_a(f(b))$;\quad $v\leftarrow z/\norm{z}$; \tcp*[r]{parent direction in local frame}
    Choose $\abs{C(a)}$ unit vectors $u_i\in\mathbb{S}^{k-1}$ from a spherical code;\\
    $R\leftarrow \Align(e_1,-v)$;\quad $u_i\leftarrow Ru_i$; \tcp*[r]{place children away from parent}
    \ForEach{$c_i\in C(a)$}{
      $y_i\leftarrow \rho\,u_i$;\quad $f(c_i)\leftarrow \MobiusBack(\phi_a,y_i)$; \tcp*[r]{$=\phi_a^{-1}(y_i)$}
    }
  }
}
\Return{$f$};
\end{algorithm}

\begin{proof}
	We construct the embedding $f$ using Algorithm~\ref{alg:generalized-sarkar-construction} and prove it satisfies the stated bi-Lipschitz guarantee under the given curvature condition. The proof proceeds in three steps: first, by invoking the geometric lemma from the analysis of Sarkar's construction; second, by applying it to our specific parameters; and third, by verifying the resulting embedding properties.
	
	\paragraph{Lemma for distortion in $\mathbb{H}_\kappa^k$.} The core of Sarkar's method \citep{sarkar_low_2011} is to place the children of a node on a hyperbolic sphere of radius $\tau$ centered at that node. To ensure the global embedding has low distortion, these children must be spaced with sufficient angular separation so that their descendant subtrees reside in nearly disjoint cones.
	
	Analysis of this construction \citep[Proposition 3.1]{de_sa_representation_2018}\footnote{\citet[Appendix D]{de_sa_representation_2018} treat hyperbolic space with unit curvature $\kappa = 1$, equation \eqref{eq:pf-achievability-via-generalized-sarkar-construction-de-sa-prop-3-1} is a straightforward adaptation to hyperbolic space with general curvature $\kappa$.}, which builds upon Sarkar's original work, yields the following sufficient condition. For a node with $\Delta$ children, if the dimensionless hyperbolic radius $\tau \sqrt{\kappa}$ satisfies
	\begin{equation}\label{eq:pf-achievability-via-generalized-sarkar-construction-de-sa-prop-3-1}
		\tau \sqrt{\kappa} \geq \frac{\alpha_k \log \Delta}{\epsilon},
	\end{equation}
	then the resulting embedding guarantees a worst-case multiplicative distortion of at most $1+\epsilon$. The constant $\alpha_k > 0$ comes from the optimal packing density of spherical caps on the $(k-1)$-sphere $\mathbb{S}^{k-1}$ and depends only on the embedding dimension $k$. Since our tree $T$ has maximum degree $\Delta$, enforcing this condition for $\Delta = \Delta$ ensures the distortion bound holds for every node in the tree.
	
	The spherical code is chosen so that the pairwise angles between the $|C(a)|\le \Delta$ directions are at least $\Omega(\epsilon/ \log \Delta)$, which is achievable in $\mathbb S^{k-1}$ under the condition \eqref{eq:pf-achievability-via-generalized-sarkar-construction-de-sa-prop-3-1}.

	For clarity we present the homogeneous case $w_e\equiv \lambda$; the weighted case follows by applying the construction to the subdivided (unweighted) tree or by using edge-dependent step sizes as in the weighted variants of Sarkar's analysis.
	
	\paragraph{Derivation of the curvature condition.} We desire a scale-normalized embedding where the intrinsic correlation distance is preserved. In our input metric, the distance from a parent to a child is exactly $\lambda$. In Algorithm~\ref{alg:generalized-sarkar-construction}, this same edge is mapped to a hyperbolic segment of length $\tau$. To achieve the unit scale factor ($s=1$), we set $\tau = \lambda$. Substituting this choice into the sufficient condition \eqref{eq:pf-achievability-via-generalized-sarkar-construction-de-sa-prop-3-1} gives:
		\begin{equation*}
			\lambda \sqrt{\kappa} \geq \frac{\alpha_k \log \Delta}{\epsilon} \implies \sqrt{\kappa} \geq \frac{\alpha_k \log \Delta}{\lambda \epsilon}.
		\end{equation*}
		
	\paragraph{Verification of the bi-Lipschitz property.} Assume the target curvature $\kappa$ satisfies the condition derived above. We execute the construction with parameters $\tau = \lambda$ and target space $\mathbb{H}_\kappa^k$.
		\subparagraph{Scale.} By construction, for any parent-child pair $(u,v)$, the hyperbolic distance is $\dH(f(u), f(v)) = \tau = \lambda$. Since the correlation distance is also $d_{\mathrm{corr}}(u,v) = \lambda$, the local scaling is exactly 1.
		\subparagraph{Distortion.} Because the curvature condition ensures the angular separation requirement is met, the global embedding is guaranteed to have distortion at most $1+\epsilon$. Formally, for all pairs $x,y \in V$:
			\begin{equation*}
				\frac{1}{1+\epsilon} d_{\mathrm{corr}}(x,y) \leq \dH(f(x),f(y)) \leq (1+\epsilon) d_{\mathrm{corr}}(x,y).
			\end{equation*}
			
	Thus, $f$ is indeed a $(1+\epsilon)$-bi-Lipschitz embedding.
\end{proof}

\subsection{Galton-Watson Packing Converse}\label{app:gw-packing-converse}

\begin{theorem}[Galton-Watson packing converse]\label{thm:gw-packing-converse} Assume the homogeneous regime $w_e\equiv \lambda$, so $d_{\mathrm{corr}}(u,v)=\lambda\abs{\mathrm{path}(u,v)}$. Fix a distortion $D \ge 1$ and scale $s > 0$. On the event $\rge$ (where the tree grows exponentially), if there exists an $(s,D)$-bi-Lipschitz embedding $f: V_{\leq R} \to \mathbb{H}^k_\kappa$, then the curvature $-\kappa$ must satisfy:
	\begin{equation*}
		\sqrt{\kappa} \geq \frac{\log m - \eta}{(k-1)sD\lambda} - o(1) \quad (as \ R\to\infty).
	\end{equation*}
\end{theorem}

\begin{proof}
	We utilize a standard volumetric packing argument. We assume the event $\rge$ holds, ensuring the tree leaves $L_R$ have realized their expected exponential abundance.
	
	\paragraph{Separation (lower bound).} Let $L_R$ be the set of nodes at depth $R$. For distinct $u,v\in L_R$, $\abs{\mathrm{path}(u,v)}\ge 2$, hence $d_{\mathrm{corr}}(u,v)\ge 2\lambda$. Under an $(s,D)$-bi-Lipschitz embedding $f$, the images must be separated by:
	\begin{equation*}
		\dH(f(u), f(v)) \geq \frac{s}{D} d_{\mathrm{corr}}(u,v) \geq \frac{2s\lambda}{D} \doteq \delta.
	\end{equation*}
	Thus, disjoint hyperbolic balls of radius $\delta/2$ can be placed around each leaf image.
	
	\paragraph{Containment (upper bound).} For any $u\in L_R$, $\abs{\mathrm{path}(r,u)}=R$ so $d_{\mathrm{corr}}(r,u)=\lambda R$. The embedding puts $f(u)$ within distance $sD\lambda R$ of the root image. Let $L_{\mathrm{max}} \doteq sD\lambda R$. All disjoint leaf balls are contained within a ball of radius $L_{\mathrm{max}} + \delta/2$.
	
	\paragraph{Volume argument.} The sum of the volumes of the small balls must be less than the volume of the containing ball:
	\begin{equation*}
		\abs{L_R} \cdot \VolH(\delta/2) \leq \VolH(L_{\mathrm{max}} + \delta/2).
	\end{equation*}
	In hyperbolic space $\Hkk$ of curvature $-\kappa$, the volume of a ball of radius $r$ satisfies:
	\begin{equation*}
		\VolH(r) = \Omega_k \int_0^r \sinh^{k-1}(\sqrt{\kappa} t) \, dt \leq C_k \kappa^{-(k-1)/2} \exp((k-1)\sqrt{\kappa} r),
	\end{equation*}
	where $C_k$ depends only on dimension. Crucially, the dependence on $\kappa$ in the prefactor is polynomial ($\kappa^{-(k-1)/2}$).
	
	On $\rge$, $\abs{L_R} \geq \exp((\log m - \eta)R)$. Taking logarithms of the volume inequality:
	\begin{align*}
		(\log m - \eta)R + \log \VolH(\delta/2) &\leq \log\left(C_k \kappa^{-(k-1)/2}\right) + (k-1)\sqrt{\kappa}\left(L_{\mathrm{max}} + \frac{\delta}{2}\right).
	\end{align*}
	Substituting $L_{\mathrm{max}} = sD\lambda R$:
	\begin{equation*}
		(\log m - \eta)R \leq (k-1)sD\lambda \sqrt{\kappa} R + O(1) + O(\sqrt{\kappa})
.
	\end{equation*}
	The prefactors involving $\kappa$ and $\delta$ are constant with respect to $R$ (or scale as $\sqrt{\kappa}$). For sufficiently large $R$, the linear terms in $R$ dominate, yielding:
	\begin{equation*}
		\sqrt{\kappa} \geq \frac{\log m - \eta}{(k-1)sD\lambda} - o(1).
	\end{equation*}
\end{proof}

\subsection{Lipschitz Preservation}\label{app:lipschitz-preservation}

\begin{lemma}[Lipschitz preservation under bi-Lipschitz embeddings]\label{lem:lipschitz-preservation-appendix}
Let $\phi:(V,d_{\mathrm{corr}})\to(\Hkk,\dH)$ be $D$-bi-Lipschitz. If $g:V\to[-1,1]$ is $L$-Lipschitz with respect to $d_{\mathrm{corr}}$, then the induced function $\hat g$ on $\phi(V)$ defined by $\hat g(\phi(v))=g(v)$ is $(LD)$-Lipschitz with respect to $\dH$.
\end{lemma}

\begin{proof}
	For any $x = \phi(u), y = \phi(v) \in \phi(V)$:
	\begin{align*}
		\abs{\hat{g}(x) - \hat{g}(y)} &= \abs{g(u) - g(v)} \\
		&\leq L \dT(u, v) \\
		&\leq L D \cdot \dH(\phi(u), \phi(v)) \tag{bi-Lipschitz lower bound} \\
		&= LD \dH(x, y).
	\end{align*}
\end{proof}

\subsection{Geometric Separation in Hyperbolic Embeddings}\label{app:hyperbolic-geometric-separation}

We collect the geometric lemmas used to establish realizability in Theorem~\ref{thm:sample-complexity-upper-bound}.

\begin{lemma}[Cone separation in the generalized Sarkar embedding]\label{lem:cone-separation}
	Let $f:V\to \mathbb{H}^k_\kappa$ be the embedding produced by the generalized Sarkar construction (Appendix~\ref{app:generalized-sarkar-construction}) with parameters chosen so that the embedding distortion is at most $1+\varepsilon$. Fix any internal node $a$ and two distinct children $c\neq c'$ of $a$. Then there exists a constant $\gamma=\gamma(k,\varepsilon,\kappa,\tau)>0$ and a totally geodesic hypersurface $H_{a,c}$ passing through $f(a)$ such that:
	\begin{equation*}
		d_{\mathbb{H}}(f(x),H_{a,c})\geq \gamma\ \forall x\in L_R^{(c)},
		\qquad
		d_{\mathbb{H}}(f(x),H_{a,c})\geq \gamma\ \forall x\in L_R^{(c')},
	\end{equation*}
	and the two sets lie on opposite sides of $H_{a,c}$ (i.e., the signed distance to $H_{a,c}$ has opposite signs). In particular, the child-subtrees of $a$ are separated by a constant hyperbolic margin.
\end{lemma}

\begin{proof}
	Children are placed on a hyperbolic sphere of radius $\tau$ around $f(a)$ with angular separation enforced by the spherical code. Under the curvature condition used to obtain distortion $1+\varepsilon$, the descendant sets of different children remain confined to disjoint cones based at $f(a)$, and the bisector hypersurface between the corresponding cone axes yields a separating totally geodesic hypersurface with constant margin. See the geometric separation argument underlying \citet{sarkar_low_2011} and \citet[Proposition 3.1]{de_sa_representation_2018}.
\end{proof}

\begin{lemma}[Lipschitz margin classifier from a geodesic hypersurface]\label{lem:lipschitz-margin-classifier}
	Let $H\subset\mathbb{H}^k_\kappa$ be a totally geodesic hypersurface and $\gamma>0$.
	Define
	\begin{equation*}
		g_H(x)\doteq \mathrm{clip}\left(\frac{1}{\gamma}\mathrm{sdist}(x,H)\right)\in[-1,1],
	\end{equation*}
	where $\mathrm{sdist}(\cdot,H)$ is the signed distance to $H$ and $\mathrm{clip}(t)=\max(-1,\min(1,t))$.
	Then $g_H$ is $(1/\gamma)$-Lipschitz with respect to $\dH$.
	Moreover, if two sets lie on opposite sides of $H$ with margin at least $\gamma$, then $g_H$ separates them with zero classification error under threshold at $0$.
\end{lemma}

\begin{proof}
	Distance to a closed set is 1-Lipschitz in any metric space; the signed distance to a totally geodesic hypersurface is also 1-Lipschitz along geodesics, and clipping is 1-Lipschitz on $\mathbb{R}$. Thus $\Lip(g_H)\le 1/\gamma$. The margin condition ensures $\mathrm{sdist}(x,H)\geq \gamma$ on one side and $\le -\gamma$ on the other, hence $g_H(x)=\pm 1$ on the two sets.
\end{proof}

\section{Information-Theoretic Lower Bounds}\label{app:fano-details}

We provide the complete proof of Theorem~\ref{thm:information-theoretic-lower-bound} under the canonical protocol (Section~\ref{subsec:learning-protocol}). The proof is most transparent in a stronger oracle model that reveals the depth-$I$ child index of the sampled leaf. A lower bound in this stronger model applies a fortiori to the original learning problem.

\paragraph{Oracle model.} Fix $m\ge 2$, depth $R\ge 1$, and noise $\rho\in(0,1/2)$. The unknown parameter is $\theta\in[m]^R$. Each sample consists of:
\begin{enumerate}
	\item $I\sim \mathrm{Unif}([R])$.
	\item $C\sim \mathrm{Unif}([m])$ (interpreted as the child index $\mathrm{ch}_I(v)$ of the sampled leaf at depth $I$).
	\item $Y\in\{0,1\}$ drawn from a BSC($\rho$) with input $\mathbf{1}\{C=\theta_I\}$:
		\begin{equation*}
			\Pr_\theta[Y=1\mid I=i,C=c]= \begin{cases}
				1-\rho & c=\theta_i,\\
				\rho & c\neq \theta_i.
			\end{cases}
		\end{equation*}
\end{enumerate}
Intuitively, conditioning on depth $i$, the oracle tells us which of the $m$ child-subtrees the example came from, and we receive a noisy bit indicating whether it is the distinguished child indexed by $\theta_i$.

\paragraph{Risk proxy.} Any estimator $\hat\theta$ with $\Pr[\hat\theta\neq \theta]\ge 1/2$ has Hamming error $\Omega(R)$ on a standard Varshamov-Gilbert packing, hence constant excess risk for any loss that penalizes misidentifying a constant fraction of coordinates.

In particular, for the path-identification loss used in the main text (average error across depths), misidentifying $\Omega(R)$ coordinates implies constant excess risk, since each incorrect coordinate contributes a constant gap in conditional prediction error under $\rho\in(0,1/2)$.

\begin{lemma}[Varshamov-Gilbert packing]\label{lem:vg-packing}
	There exists $\Theta_0\subset[m]^R$ with $\log\abs{\Theta_0}\asymp R\log m$ and pairwise Hamming distance at least $R/8$.
\end{lemma}

\begin{lemma}[Per-sample KL bound]\label{lem:kl-bound}
	Let $P_\theta$ denote the joint law of $(I,C,Y)$ under parameter $\theta$. Then for any $\theta,\theta'\in[m]^R$,
	\begin{equation*}
		D_{\mathrm{KL}}(P_\theta\|P_{\theta'})  \le\ \frac{\dHam(\theta,\theta')}{R}\frac{2}{m} \beta_{\rho},
	\end{equation*}
	where $\beta_{\rho}\doteq D_{\mathrm{KL}}(\mathrm{Bern}(1-\rho)\|\mathrm{Bern}(\rho))$.
\end{lemma}

\begin{proof}
	Condition on $I=i$. If $\theta_i\neq\theta'_i$, the conditional distribution of $Y$ differs only when $C\in\{\theta_i,\theta'_i\}$, an event of probability $2/m$. On that event, $Y$ is $\mathrm{Bern}(1-\rho)$ under one parameter and $\mathrm{Bern}(\rho)$ under the other, so the conditional KL contribution is at most $\beta_\rho$. Hence
	\begin{equation*}
		D_{\mathrm{KL}}(P_\theta(\cdot\mid I=i)\|P_{\theta'}(\cdot\mid I=i))\le \frac{2}{m}\beta_\rho.
	\end{equation*}
	Averaging over $I\sim\mathrm{Unif}([R])$ introduces the factor $\dHam(\theta,\theta')/R$.
\end{proof}

We are now ready to prove Theorem~\ref{thm:information-theoretic-lower-bound}.

\begin{proof}
Let $Z=(I,C,Y)$ denote a single oracle sample and let $Z^n$ be $n$ i.i.d. copies. Draw $\theta$ uniformly from $\Theta_0$ and observe $Z^n$ consisting of $n$ i.i.d. oracle samples. By Fano's inequality,
\begin{equation*}
	\Pr[\hat\theta\neq \theta]\ \ge\ 1-\frac{I(\theta;Z^n)+\log 2}{\log\abs{\Theta_0}}.
\end{equation*}

Let $\bar P \doteq \frac{1}{|\Theta_0|}\sum_{\theta\in\Theta_0} P_\theta$ denote the mixture distribution. Then
\begin{equation*}
	I(\theta;Z^n)=\frac{1}{|\Theta_0|}\sum_{\theta\in\Theta_0} D_{\mathrm{KL}}(P_\theta^n\|\bar P^n)
\le \frac{1}{|\Theta_0|^2}\sum_{\theta,\theta'\in\Theta_0} D_{\mathrm{KL}}(P_\theta^n\|P_{\theta'}^n)
\le \max_{\theta\neq \theta'} D_{\mathrm{KL}}(P_\theta^n\|P_{\theta'}^n),
\end{equation*}

Using Lemma~\ref{lem:kl-bound},
\begin{equation*}
	I(\theta;Z^n)\ \le\ \max_{\theta\neq \theta'} D_{\mathrm{KL}}(P_\theta^n\|P_{\theta'}^n)\ \le\ n \frac{2}{m}\beta_\rho.
\end{equation*}
Since $\log\abs{\Theta_0}\asymp R\log m$, we obtain $\Pr[\hat\theta\neq \theta]\ge 1/2$ whenever
\begin{equation*}
	n\ \le\ c\cdot \frac{m}{\beta_\rho}R\log m
\end{equation*}
for a sufficiently small universal constant $c>0$.
Thus $n=\Omega(mR\log m)$ samples are necessary in the oracle model. Since the oracle reveals $(I,C)$ whereas the original protocol reveals only $(I,\phi(V))$, any estimator in the original model induces an estimator in the oracle model by composition; thus a lower bound in the oracle model applies a fortiori.
\end{proof}

\section{Low-Rank Representation Details}\label{app:spectral-barrier-details}

\subsection{Connection to Edge-Cut Labels}\label{app:connection-to-edge-cut-labels}

For binary trees ($m = 2$), the tree-Haar wavelets are directly related to edge-cut labels.

\begin{definition}[Edge-cut labels]
	For an internal node $v$ with children $v_L, v_R$, define the edge-cut label $y_v \in \{-1, 0, +1\}^N$ by
	\begin{equation*}
		y_v(x) = \begin{cases}
			+1 & \text{if } x \in S_{v_L}, \\
			-1 & \text{if } x \in S_{v_R}, \\
			0 & \text{if } x \notin S_v.
		\end{cases}
	\end{equation*}
\end{definition}

\begin{lemma}[Haar wavelets from edge cuts]
	For $m = 2$, the tree-Haar wavelet at node $v$ is
	\begin{equation*}
		\psi_v = \frac{1}{\sqrt{\abs{S_v}}} y_v.
	\end{equation*}
\end{lemma}

\begin{corollary}
For $m=2$, approximating all edge-cut labels $\{y_v\}$ to squared $\ell_2$ error $\le \acc$ is equivalent (up to the deterministic normalization $\psi_v = y_v/\sqrt{\abs{S_v}}$) to approximating all Haar wavelets $\{\psi_v\}$ to squared $\ell_2$ error $\le \acc/\abs{S_v}$ at node $v$. In particular, any class that approximates a constant fraction of these contrasts to constant error must have dimension $k=\Omega(N)$.
\end{corollary}

\subsection{Orthonormality Proof}\label{app:orthonormality-details}

Throughout this subsection we consider the balanced $m$-ary tree with $N=m^R$ leaves.

\begin{lemma}[Orthonormality and completeness]\label{lem:orthonormality-and-completeness}
	$\mathcal W$ is an orthonormal set with $\abs{\mathcal W} = N-1$, forming a basis of $\mathbf{1}^\perp \subset \R^N$.
\end{lemma}
\begin{proof}
For $m=2$ take $a=(1/\sqrt2,-1/\sqrt2)$\footnote{The proof for $m \neq 2$ is normalized analogously.}; substituting into Definition~\ref{def:tree-haar-wavelets} yields $\psi_v(x)=\pm 1/\sqrt{\abs{S_v}}$ on $S_{v_L},S_{v_R}$, hence $\psi_v = y_v/\sqrt{\abs{S_v}}$. Let $(v, j)$ and $(v', j')$ be distinct index pairs.

	\paragraph{Unit norm.} By construction:
	\begin{equation*}
		\norm{\psi_{v,j}}^2 = \sum_{i=1}^m \sum_{x \in S_{v,i}} \frac{(a_i^{(j)})^2}{\abs{S_v}/m} = \sum_{i=1}^m (a_i^{(j)})^2 = \norm{a^{(j)}}^2 = 1.
	\end{equation*}

	\paragraph{Orthogonality.} Consider three cases:

		\subparagraph{Disjoint supports.} ($S_v \cap S_{v'} = \emptyset$). Then $\langle \psi_{v,j}, \psi_{v',j'} \rangle = 0$ trivially.

		\subparagraph{Same node.} ($v = v'$, $j \neq j'$). Then:
		\begin{equation*}
			\langle \psi_{v,j}, \psi_{v,j'} \rangle = \sum_{i=1}^m \frac{a_i^{(j)} a_i^{(j')}}{\abs{S_v}/m} \abs{S_{v,i}} = \sum_{i=1}^m a_i^{(j)} a_i^{(j')} = \langle a^{(j)}, a^{(j')} \rangle = 0.
		\end{equation*}

		\subparagraph{Ancestor-descendant.} ($S_{v'} \subset S_v$). Then $v'$ is contained in some child $S_{v,i}$ of $v$, and $\psi_{v,j}$ is constant on $S_{v,i}$:
		\begin{equation*}
			\langle \psi_{v,j}, \psi_{v',j'} \rangle = \frac{a_i^{(j)}}{\sqrt{\abs{S_v}/m}} \sum_{x \in S_{v'}} \psi_{v',j'}(x) = \frac{a_i^{(j)}}{\sqrt{\abs{S_v}/m}} \cdot 0 = 0,
		\end{equation*}
		since $\psi_{v',j'}$ has zero sum on $S_{v'}$ because $a^{(j')}\perp \mathbf{1}$.

	\paragraph{Counting.} The number of internal nodes at depth $d$ is $m^d$, each contributing $m-1$ wavelets:
	\begin{equation*}
		\abs{\mathcal{W}} = \sum_{d=0}^{R-1} m^d (m-1) = (m-1) \frac{m^R - 1}{m - 1} = m^R - 1 = N - 1.
	\end{equation*}

	Since each $\psi_{v,j}$ has zero mean (lies in $\mathbf{1}^\perp$) and $\abs{\mathcal{W}} = \dim(\mathbf{1}^\perp) = N-1$, the set $\mathcal{W}$ is an orthonormal basis of $\mathbf{1}^\perp$.
\end{proof}

\section{Beyond Trees: High-Temperature Extensions}\label{app:beyond-trees-high-temperature-extensions}

Real hierarchical datasets often contain lateral edges. We show that under high-temperature conditions, the main results extend.

\subsection{High-Temperature Correlation Bounds}\label{app:high-temperature-correlation-bounds}

Let $G = (V, E)$ be a graph with maximum degree $\Delta$, equipped with an Ising model.

Assume a ferromagnetic Ising model with $J_e\ge 0$ and zero external field. Define:
\begin{equation*}
	t_e = \tanh(J_e), \quad t_{\min} = \min_e t_e,\quad t_{\max} = \max_e t_e,\quad \alpha = (\Delta - 1) t_{\max}.
\end{equation*}

\begin{theorem}[Two-sided correlation bounds]\label{thm:two-sided-correlation-bounds}
	If $\alpha < 1$, then for all $u \neq v$:
	\begin{equation*}
		(t_{\min})^{d_G(u,v)} \leq \langle X_u X_v \rangle_G \leq C_{\Delta,\alpha} \alpha^{d_G(u,v)},
	\end{equation*}
	where $C_{\Delta,\alpha} = \Delta / ((\Delta-1)(1-\alpha))$.
\end{theorem}

\begin{proof}
	We treat the lower and upper bound separately.
	
	\paragraph{Lower bound.} By FKG monotonicity \citep{cmp/1103857443}, $\langle X_u X_v \rangle_G \geq \langle X_u X_v \rangle_H$ where $H$ is any geodesic path. By tree factorization, $\langle X_u X_v \rangle_H \geq (t_{\min})^{d_G(u,v)}$.

	\paragraph{Upper bound.} Under $\alpha<1$, the standard high-temperature (cluster/SAW) expansion converges and yields an absolutely summable path representation:
	\begin{equation*}
		\langle X_u X_v \rangle_G \leq \sum_{\gamma \in \mathrm{SAW}(u \to v)} \prod_{e \in \gamma} t_e \leq \sum_{L \geq d_G(u,v)} \Delta(\Delta-1)^{L-1} t_{\max}^L = \frac{\Delta \alpha^{d_G(u,v)}}{(\Delta-1)(1-\alpha)}.
	\end{equation*}
	Here $\mathrm{SAW}$ is the set of self-avoiding walks from $u$ to $v$.
\end{proof}

\begin{corollary}[Quasi-isometry to graph distance]
Under the conditions of Theorem~\ref{thm:two-sided-correlation-bounds} (so $\langle X_uX_v\rangle_G>0$):
	\begin{equation*}
		a \cdot d_G(u,v) - b \leq d_{\mathrm{corr}}(u,v) \leq A \cdot d_G(u,v),
	\end{equation*}
	where $A = -\log t_{\min}$, $a = -\log \alpha$, $b = \log C_{\Delta,\alpha}$.
\end{corollary}

\subsection{Robustness to Lateral Connections (Tree-Like Graphs)}\label{app:robustness-to-lateral-connections}

Real-world hierarchies often exhibit lateral connections. We show that our results extend to any graph that is a bounded-range perturbation of a tree.

\begin{definition}[$K$-local extension]
Let $T=(V,E_T)$ be a rooted tree, and let $\dT(\cdot,\cdot)$ denote the unweighted graph distance on $T$ (edge count along the unique path). A graph $G=(V,E_G)$ is a $K$-local extension of $T$ if $E_T\subseteq E_G$ and for every lateral edge $(u,v)\in E_G\setminus E_T$ we have $\dT(u,v)\le K$.
\end{definition}
	This covers sibling augmentation ($K=2$), cousin connections ($K=4$), and local grid-like structures attached to the hierarchy.

\begin{lemma}[Quasi-isometry of tree-like graphs]
	Let $G$ be a $K$-local extension of $T$. Then for all $u,v \in V$:
	\begin{equation*}
		\frac{1}{K} \dT(u,v) \leq d_G(u,v) \leq \dT(u,v).
	\end{equation*}
\end{lemma}
\begin{proof}
	We shall again treat the upper and lower bound separately.
	\paragraph{Upper bound.} Since $E_T \subseteq E_G$, any path in $T$ is a valid path in $G$. Thus $d_G \leq \dT$.
	
	\paragraph{Lower bound.} Consider a shortest path $P_G$ between $u$ and $v$ in $G$. Replace every lateral edge $e=(x,y)$ in $P_G$ with the unique path in $T$ between $x$ and $y$. By definition, this tree-path has length $\dT(x,y) \leq K$. The resulting path lies entirely in $T$ and has length at most $K\mathrm{length}(P_G)=K d_G(u,v)$.
 Thus $\dT(u,v) \leq K d_G(u,v)$.
\end{proof}

\begin{corollary}[Extension of results]
	If the graph admits a tree backbone $T$ such that (i) correlations induce a metric quasi-isometric to $\dT$, and (ii) the supervised task family is defined w.r.t. $T$, then the bounds carry through with constants depending on $K$ and the correlation-decay parameters:
	\begin{enumerate}
		\item \emph{Euclidean representation lower bound.} If $d_{\mathrm{corr}}$ is quasi-isometric to the tree backbone distance $\dT$, then the Euclidean collision argument applies with constants distorted by the quasi-isometry parameters. In particular, a cut that is $O(1/R)$-Lipschitz in $\dT$ remains $O(1/R)$-Lipschitz in $d_{\mathrm{corr}}$ up to constants, and the resulting Euclidean Lipschitz lower bound remains exponential in $R/k$ (up to $K$-dependent constants in the exponent).
		\item \emph{Rank lower bound.} The wavelet/trace argument depends only on the leaf partition induced by the tree backbone and is unaffected by adding lateral edges on the same vertex set.
		\item \emph{Hyperbolic embedding.} If $d_G$ and $\dT$ are bi-Lipschitz equivalent up to factor $K$ as above, then composing this equivalence with a $D$-bi-Lipschitz embedding of $(V,\dT)$ into $\mathbb H^k_\kappa$ yields an $O(KD)$-distortion embedding of $(V,d_G)$ into $\mathbb H^k_\kappa$.
	\end{enumerate}
\end{corollary}

\end{document}